\documentclass{article} 
\usepackage{iclr2015,times}
\usepackage{url}

\usepackage{graphicx} 
\usepackage{verbatim}
\usepackage{amsmath}
\usepackage{amssymb}
\usepackage{amsthm}

\usepackage{tikz}
\usetikzlibrary{positioning}
\usetikzlibrary{arrows, decorations.markings}
\input{arrowsnew}

\usepackage{rotating}

\usepackage{natbib}
\bibpunct[; ]{(}{)}{;}{a}{}{;}

\definecolor{bl}{RGB}{20,20,150}
\usepackage[breaklinks]{hyperref}

\usepackage{breakurl}

\newcommand{\Gcal}{\mathcal{G}}
\newcommand{\Mcal}{\mathcal{M}}

\newcommand{\DBM}{\operatorname{DBM}}
\newcommand{\RBM}{\operatorname{RBM}}
\newcommand{\FF}{\operatorname{FF}}

\newcommand{\bfi}{{\mathbf{i}}}
\newcommand{\bfo}{{\mathbf{o}}}
\newcommand{\bfx}{{\mathbf{x}}}
\newcommand{\bfz}{{\mathbf{z}}}
\newcommand{\bfW}{{\mathbf{W}}}
\newcommand{\bfB}{{\mathbf{b}}}

\usepackage{tcolorbox}
\makeatletter
\newcommand{\mybox}[1]{%
	\setbox0=\hbox{#1\!\!}%
	\setlength{\@tempdima}{\dimexpr\wd0+13pt}%
	\begin{tcolorbox}[colframe=black, boxrule=1pt, arc=4pt,
		left=6pt,right=6pt, top=2pt,bottom=2pt,boxsep=0pt,width=\@tempdima]
		#1
	\end{tcolorbox}
}

\newcommand{\unitlayer}[4]{ 
	\setlength\fboxsep{1pt}
	\mybox{\tikzset{node distance=.4cm, auto}
		\begin{tikzpicture}[scale=0.9, every node/.style={transform shape}]
		\tikzstyle{neuron}=[circle, line width=.75pt, draw=black, inner sep=.025cm, minimum size = .65cm, fill=bl!10]
		
		\foreach \name / \i in {1,...,#2}
		\node[neuron] (I-\name) at (\i,0) {$#1_{#4,\i}$};
		
		\node (I-dots) [node distance = .8cm, right of = I-#2] {$\cdots$};
		\node[neuron] (I-end)  [node distance = .8cm, right of = I-dots] {$#1_{#4,#3}$};
		
		\end{tikzpicture}
	}}
	
	\tikzstyle{vecArrow} = [thick, decoration={markings,mark=at position
		1 with {\arrow[semithick]{open triangle 60}}},
	double distance=2.4pt, shorten >= 5.5pt,
	preaction = {decorate},
	postaction = {draw,line width=2.4pt, white,shorten >= 4.5pt}]
	
	\tikzstyle{vecArrowGray} = [thick, decoration={markings,mark=at position
		1 with {\arrow[semithick]{open triangle 60}}},
	double distance=2.4pt, shorten >= 5.5pt,
	preaction = {decorate},
	postaction = {draw,line width=2.4pt, black!10,shorten >= 4.5pt}]
	
\newtheorem{theorem}{Theorem}

\newtheorem{proposition}[theorem]{Proposition}
\newtheorem{corollary}[theorem]{Corollary}
\theoremstyle{definition}

\title{Deep Narrow Boltzmann Machines \\are Universal Approximators}

\author{
Guido Mont\'ufar \thanks{http://personal-homepages.mis.mpg.de/montufar/} \\
Max Planck Institute for Mathematics in the Sciences\\
Inselstrasse 22, 04103 Leipzig, Germany \\
\texttt{montufar@mis.mpg.de} 
}

\iclrfinalcopy 

\iclrconference 

\begin{document}

\maketitle

\begin{abstract}
We show that deep narrow Boltzmann machines are universal approximators of probability distributions on the activities of their visible units, provided they have sufficiently many hidden layers, each containing the same number of units as the visible layer. We show that, within certain parameter domains, deep Boltzmann machines can be studied as feedforward networks. We provide upper and lower bounds on the sufficient depth and width of universal approximators. These results settle various intuitions regarding undirected networks and, in particular, they show that deep narrow Boltzmann machines are at least as compact universal approximators as narrow sigmoid belief networks and restricted Boltzmann machines, with respect to the currently available bounds for those models. 
\end{abstract}

\section{Introduction} 

It is an interesting question how the representational power of deep artificial neural networks, with several layers of hidden units, compares with that of shallow neural networks, 
with one single layer of hidden units. 
Furthermore, it is interesting how the representational power of layered networks compares in the cases of undirected and directed connections between the layers. 
A basic question in this respect is whether the classes of function approximators represented by the different network architectures can possibly reach any desired degree of accuracy, when endowed with sufficiently many computational units. This property, referred to as {\em universal approximation property}, has been established for a wide range of network architectures, including various kinds of shallow feedforward, shallow undirected, and deep feedforward networks, both in the deterministic and stochastic settings. 
Nevertheless, for deep narrow undirected network architectures, universal approximation has remained so far an open problem. 
In this paper  we prove that deep narrow Boltzmann machines are universal approximators, 
provided they have sufficiently many layers of hidden units, each having at least as many units as the visible layer. 

A Boltzmann machine~\citep{Ackley} is a network of stochastic binary units with undirected pairwise interactions. 
A deep Boltzmann machine (DBM)~\citep{SalHinton07} is a Boltzmann machine whose units build a stack of layers, 
where only pairs of units from subsequent layers interact, and only the units in the bottom layer are visible. 
The units within any given layer are conditionally independent, given the states of the units in the adjacent layers. 
Figure~\ref{figure:DBM} illustrates this architecture. 

Since the first appearance of DBMs, 
a number of papers have addressed various practical and theoretical aspects of these networks, especially regarding training and estimation~\citep[see][]{montavon2012deep,DBLP:journals/corr/abs-1301-3568,ChoDBM}. 
The undirected nature of DBMs leads to interesting and desirable properties, but it also brings with it challenges in training and analyzing them. 
A number of anticipated properties of DBMs still are missing formal verification. 
We prove that narrow DBMs have the universal approximation property. 
We focus on DBMs with layers of constant size. 
We note that, in order to obtain the universal approximation property, the first hidden layer must have at least the same size as the visible layer (minus one, when this is even). 
As a direct corollary of our main theorem, we obtain the universal approximation of conditional probability distributions on the activations of subsets of visible units, 
given the activations of the remaining visible units. 
Our analysis applies not only to binary units, but also to softmax (finite-valued) units. 

At an intuitive level, undirected networks are expected to be more powerful than their directed counterparts, since ``they allow information to flow both ways.'' 
Given that narrow deep belief networks (DBNs)~\citep{Hinton:2006:FLA:1161603.1161605} have the universal approximation property~\citep{Hinton:2008}, 
the natural expectation is that narrow DBMs also have the universal approximation property (DBNs can be regarded as the directed counterparts of DBMs). 
There are several reasons why this intuition is not straightforward to verify. 
While feedforward networks can be studied in a sequential way, with the output of any given layer being the input of the next layer, 
in the undirected case, each internal layer receives inputs from both the previous and the next layers. 
This renders recurrent signals between all units and complicates a sequential analysis. 
The key component of our proof lies in showing that, within certain parameter regions, a DBM can be regarded as a feedforward network. 
More precisely, the upper part of the network can ``neutralize'' the upward signals arriving from the lower part of the network, 
in such a way that the visible probability distributions represented by the entire network have the same form as those represented by a DBN. 
This allows us to analyze the representational power of DBMs in a sequential way and show that, 
in some well defined sense, DBMs are at least as powerful as DBNs. 


\citet*{AISTATS2012_MontavonBM12} have also proposed a feedforward perspective on DBMs. 
However, their motivation was different from ours, 
and they used the term ``feedforward'' to refer to a Gibbs sampling pass traversing the network in a feedforward manner, 
rather than to the structure of the joint probability distributions represented by the entire network. 
They showed, experimentally, that a DBM outputs a feedforward hierarchy of increasingly invariant representations.

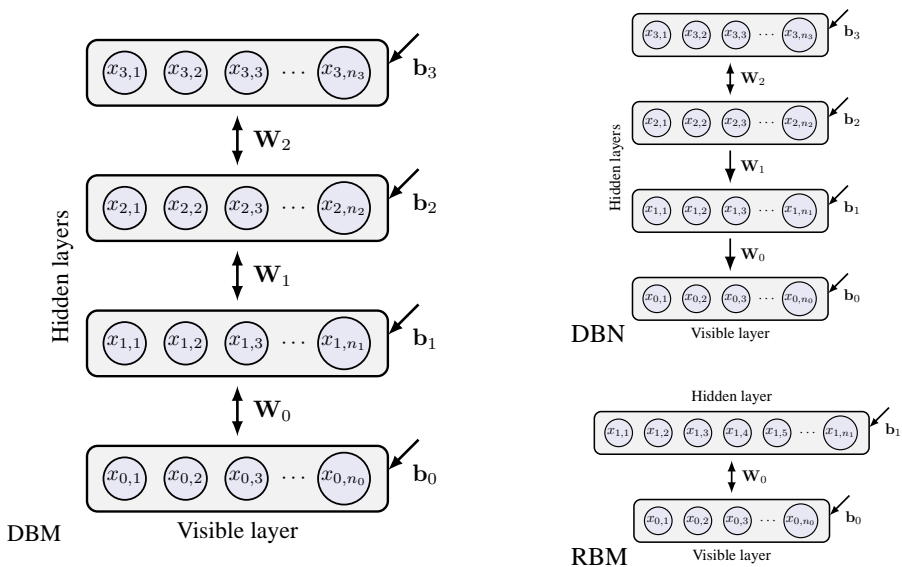
\begin{figure} 
	\centering
	\begin{tabular}{c}
		\tikzset{node distance=2cm, auto}
		\begin{tikzpicture}[scale=0.9, every node/.style={transform shape}]
		\node (H3) {\unitlayer{x}{3}{{n_3}}{3} };
		\node (H2) [below of = H3] {\unitlayer{x}{3}{{n_2}}{2} };
		\node (H1) [below of = H2] {\unitlayer{x}{3}{{n_1}}{1} };
		\node (V) [below of = H1] {\unitlayer{x}{3}{{n_0}}{0} };
		
		\draw[latexnew-latexnew, arrowhead=.2cm, line width=1pt] (H3) to node  {\;$\bfW_2$} (H2);
		\draw[latexnew-latexnew, arrowhead=.2cm, line width = 1pt] (H2) to node  {\;$\bfW_1$} (H1);
		\draw[latexnew-latexnew, arrowhead=.2cm, line width = 1pt] (H1) to node  {\;$\bfW_0$} (V);
		
		\node (h3) [right of = H3, node distance=2.1cm] {};
		\node (h2) [right of = H2, node distance=2.1cm] {};
		\node (h1) [right of = H1, node distance=2.1cm] {};
		\node (h0) [right of = V, node distance=2.1cm] {};
		\node (c3) [node distance=.7cm, right of = h3, above of =h3] {};
		\node (c2) [node distance=.7cm, right of = h2, above of =h2] {};
		\node (c1) [node distance=.7cm, right of = h1, above of =h1] {};
		\node (c0) [node distance=.7cm, right of = h0, above of =h0] {};
		\draw[-latexnew, arrowhead=.2cm,  line width = 1pt] (c3) to node {$\bfB_3$} (h3);
		\draw[-latexnew, arrowhead=.2cm,  line width = 1pt] (c2) to node {$\bfB_2$} (h2);
		\draw[-latexnew, arrowhead=.2cm,  line width = 1pt] (c1) to node {$\bfB_1$} (h1);
		\draw[-latexnew, arrowhead=.2cm,  line width = 1pt] (c0) to node {$\bfB_0$} (h0);									
		
		\node (hid) [node distance=2.5cm, left of = H1] {\begin{rotate}{90}Hidden layers\end{rotate}};
		\node (vis) [node distance=.8cm, below of = V] {Visible layer};
		\node (titl) [node distance=3cm, left of = vis] {DBM};
		\end{tikzpicture}
	\end{tabular}\qquad\quad
	\begin{tabular}{l}
		\scalebox{.65}{
			\tikzset{node distance=2cm, auto}
			\begin{tikzpicture}[scale=0.9, every node/.style={transform shape}]
			\node (H3) {\unitlayer{x}{3}{{n_3}}{3} };
			\node (H2) [below of = H3] {\unitlayer{x}{3}{{n_2}}{2} };
			\node (H1) [below of = H2] {\unitlayer{x}{3}{{n_1}}{1} };
			\node (V) [below of = H1] {\unitlayer{x}{3}{{n_0}}{0} };
			
			\draw[latexnew-latexnew, arrowhead=.225cm, line width = 1pt] (H3) to node  {\;$\bfW_2$} (H2);
			\draw[-latexnew, arrowhead=.3cm, line width = 1pt] (H2) to node  {\textcolor{black}{\;$\bfW_1$}} (H1);
			\draw[-latexnew, arrowhead=.3cm, line width = 1pt] (H1) to node  {\textcolor{black}{\;$\bfW_0$}} (V);
			
			\node (h3) [right of = H3, node distance=2.1cm] {};
			\node (h2) [right of = H2, node distance=2.1cm] {};
			\node (h1) [right of = H1, node distance=2.1cm] {};
			\node (h0) [right of = V, node distance=2.1cm] {};
			\node (c3) [node distance=.7cm, right of = h3, above of =h3] {};
			\node (c2) [node distance=.7cm, right of = h2, above of =h2] {};
			\node (c1) [node distance=.7cm, right of = h1, above of =h1] {};
			\node (c0) [node distance=.7cm, right of = h0, above of =h0] {};
			\draw[-latexnew, arrowhead=.2cm,  line width = 1pt] (c3) to node {$\bfB_3$} (h3);
			\draw[-latexnew, arrowhead=.2cm,  line width = 1pt] (c2) to node {$\bfB_2$} (h2);
			\draw[-latexnew, arrowhead=.2cm,  line width = 1pt] (c1) to node {$\bfB_1$} (h1);
			\draw[-latexnew, arrowhead=.2cm,  line width = 1pt] (c0) to node {$\bfB_0$} (h0);									
			
			\node (hid) [node distance=2.5cm, left of = H1] {\begin{rotate}{90}Hidden layers\end{rotate}};
			\node (vis) [node distance=.8cm, below of = V] {Visible layer};
			\node (titl) [node distance=3cm, left of = vis] {\scalebox{1.66}{DBN}};
			\end{tikzpicture}
		}\\
		\\
		\scalebox{.65}{
			\tikzset{node distance=2cm, auto}
			\begin{tikzpicture}[scale=0.9, every node/.style={transform shape}]
			\node (H1) {\unitlayer{x}{5}{{n_1}}{1} };
			\node (V) [below of = H1] {\unitlayer{x}{3}{{n_0}}{0} };
			
			\draw[latexnew-latexnew, arrowhead=.225cm, line width = 1pt] (H1) to node  {\;$\bfW_0$} (V);
			
			\node (h1) [right of = H1, node distance=3cm] {};
			\node (h0) [right of = V, node distance=2.1cm] {};
			\node (c1) [node distance=.7cm, right of = h1, above of =h1] {};
			\node (c0) [node distance=.7cm, right of = h0, above of =h0] {};
			\draw[-latexnew, arrowhead=.2cm,  line width = 1pt] (c1) to node {$\bfB_1$} (h1);
			\draw[-latexnew, arrowhead=.2cm,  line width = 1pt] (c0) to node {$\bfB_0$} (h0);									
			
			\node (hid) [node distance=.8cm, above of = H1] {Hidden layer};
			\node (vis) [node distance=.8cm, below of = V] {Visible layer};
			\node (titl) [node distance=3cm, left of = vis] {\scalebox{1.66}{RBM}};
			\end{tikzpicture}
		} 
	\end{tabular}
	\caption{The left panel illustrates the architecture of a DBM with a visible layer of $n_0$ units and three hidden layers of $n_1$, $n_2$, $n_3$ units. 
		Pairs of units form consecutive layers are undirectedly connected. 
		There are no connections between units from the same layer nor between units from non-consecutive layers. 
		The right panel shows the architectures of a DBN and an RBM,  
		which are the directed and the shallow versions of DBMs. 
	}\label{figure:DBM}
\end{figure}

In the remainder of the introduction we comment on a few results that appear helpful for contextualizing the present paper. 
From the network architectures mentioned above (deep, shallow, directed, undirected), 
the most extensively studied ones are the shallow feedforward networks (with one single layer of hidden units). 
A shallow feedforward network is understood as a composition of simple computational units, all having the same inputs. 
It is well known that, by tuning the parameters of the individual units, 
these networks can approximate any function on the set of inputs arbitrarily well,\footnote{Meant are reasonably well behaved functions and reasonable measures of approximation.} 
provided they have sufficiently many hidden units~\citep{Hornik89,Cybenko}. 
In other words, any function can be written, approximately, as a superposition (e.g., linear combination) of simple functions. 
This universal approximation property has been established under very general conditions both on the type of units and the type of functions being approximated~\citep[see, e.g.,][]{leshno1993multilayer,ChenChen1995}. 
See also~\citep{Barron1993,Burger2001235} for works addressing the accuracy of the approximations. 
An interesting recent example are shallow feedforward networks with {\em maxout} units~\citep{DBLP:conf/icml/GoodfellowWMCB13}. 
Besides from standard functions, i.e., deterministic output assignments given the inputs, 
shallow feedforward networks are also capable of approximating stochastic functions arbitrarily well, i.e., probabilistic output assignments given the inputs, 
when constructed with sufficiently many stochastic units. 

Deep neural networks have seen exceptional success in applications in recent years. 
Aiming at a better understanding and development of this success, 
a number of recent papers have addressed the theory of deep architectures~\citep[see][]{BengioExprDeep,journals/jmlr/Baldi12,pascanu2013number,montufar2014number}. 
It is not so long ago that~\citet{Hinton:2008} investigated deep belief networks (DBNs) \citep{Hinton:2006:FLA:1161603.1161605} with narrow layers of stochastic binary units (all having about the same number of units). 
They showed that these architectures can approximate any binary probability distribution on the states of their visible units arbitrarily well, provided the number of hidden layers is large enough (exponential in the number of visible units). 
The minimal depth of universal approximators of this kind has been studied subsequently in~\citep{LeRoux2010,Montufar2011,Montufar2014dbn}. 
The approximation properties of DBNs with real-valued visible units and binary hidden units have been treated in recent work as well~\citep{conf/icml/KrauseFGI13}. 

Boltzmann machines~\citep*{hinton:optimal,Ackley,Hinton:1986:LRB:104279.104291} are energy based models describing the statistical behavior of pairwise interacting stochastic binary units. 
They have roots in statistical physics and have been studied intensively as special types of graphical probability models and exponential families. 
In particular, information geometry has provided deep geometric insights about learning and approximation of probability distributions by this kind of networks~\citep{125867}. 
It is well known that Boltzmann machines are universal approximators of probability distributions over the states of their visible units, provided they have sufficiently many hidden units~\citep[see][]{194417,Younes1996109}. 
The situation is more differentiated when a specific structure is imposed on the network, e.g., a layered structure, where only pairs of units in subsequent layers may be connected. 
This imposes non-trivial restrictions on the sets of representable distributions. 
For the shallow layered version of the Boltzmann machine, called restricted Boltzmann machine (RBM)~\citep{Smolensky1986,Freund1992}, 
the universal approximation capability has been shown in~\citep{Freund1992,Younes1996109,LeRoux2008}, provided the hidden layer is large enough 
(having exponentially more units than the visible layer). 
More recently, questions related to the minimal number of hidden units that is sufficient for universal approximation by RBMs have been studied in~\citep{LeRoux2008,Montufar2011,NIPS2011_0307,montufar2013discrete,NIPS2013_5020}.  
Nonetheless, universal approximation results for the deep versions of RBMs, the deep Boltzmann machines (DBMs)~\citep{SalHinton07}, 
have been missing so far (except when the hidden layers have exponentially many more units than the visible layer). \\

This paper is organized as follows. 
In Section~\ref{sec:definitions} we give necessary definitions and fix notations. 
In Section~\ref{sec:universalapproximation} we present our main result: the universal approximation property of narrow DBMs. 
The proof of this result is elaborated in Sections~\ref{sec:compositional} and ~\ref{sec:analysis}. 
In Section~\ref{sec:compositional} we address the compositional structure of DBMs. We express the probability distributions represented by a DBM in terms of the probability distributions represented by two smaller DBMs and a feedforward layer with shared parameters. 
In Section~\ref{sec:analysis} we elaborate an approach to study DBMs from a feedforward perspective. 
We first present a trick to effectively disentangle the shared parameters between intermediate marginal distributions and lower conditional distributions. 
This is followed by a feedforward analysis proving the universal approximation property. 
In Section~\ref{sec:discussion} we offer a discussion of the result.

\section{Definitions}
\label{sec:definitions}

In this section we fix notation and technical details. 
A deep Boltzmann machine with $L+1$ layers of $n_0,n_1,\ldots, n_L$ units is a model of joint probability distributions of the form 
\begin{multline}
p_{\bfW,\bfB}(\bfx_0,\bfx_1\ldots,\bfx_L) = \frac{1}{Z(\bfW,\bfB)}\exp( \sum_{l=0}^{L-1}  \bfx_{l}^\top \bfW_{l} \bfx_{l+1}  + \sum_{l=0}^L \bfx_l^\top {\bfB_l}), \\
\quad\text{for all $(\bfx_0^\top ,\ldots, \bfx_L^\top )^\top \in\{0,1\}^{n_0+\cdots+n_L}$} . 
\label{eq:defjoint}
\end{multline}
Here  $\bfx_l=(x_{l,1},\ldots, x_{l,n_l} )^\top \in\{0,1\}^{n_l}$ denotes the joint state of the units in the $l$-th layer and 
$(\bfx_0^\top ,\ldots, \bfx_L^\top )^\top\in\{0,1\}^N$, $N=\sum_{l=0}^L n_l$, the joint state of all units. 
The parameters of this model are 
$\bfW=\{ \bfW_0,\ldots, \bfW_{L-1}\}$ and $\bfB=\{ \bfB_0,\ldots,\bfB_L\}$, where $\bfW_l \in \mathbb{R}^{n_l\times n_{l+1}}$ is a matrix of interaction weights between units from the $l$-th and $(l+1)$-th layers, for $l=0,\ldots, L-1$, and $\bfB_l\in\mathbb{R}^{n_l}$ is a vector of biases for the units in the $l$-th layer, for $l=0,\ldots,L$.  
The function $Z(\bfW,\bfB)$ is defined in such a way that the entries of $p_{\bfW,\bfB}$ add to one.  

The set of all probability distributions of the form~\eqref{eq:defjoint}, for all choices of $\bfW$ and $\bfB$, 
is a smooth manifold (an exponential family) of dimension $\sum_{l=0}^{L-1}n_l n_{l+1} + \sum_{l=0}^L n_l$. 
This manifold is embedded in the $(2^N-1)$-dimensional set $\Delta_{N}$ of all possible probability distributions over $\{0,1\}^N$. 
%
Note that every probability distribution of the form~\eqref{eq:defjoint} is strictly positive, 
meaning that it assigns strictly positive probability to every state. 
We denote this model by $\DBM_{n_0,\ldots,n_L}$, or $\DBM$, for simplicity, when $n_0,\ldots, n_L$ are clear. 

The marginal probability distributions over the joint states $\bfx_0$ of the units in the bottom layer are obtained by marginalizing out $\bfx_1, \ldots, \bfx_L$: 
\begin{equation}
p_{\bfW,\bfB}(\bfx_0) = \sum_{\bfx_1,\ldots, \bfx_L} p_{\bfW,\bfB}(\bfx_0,\bfx_1,\ldots, \bfx_L), \quad\text{for all $\bfx_0 \in\{0,1\}^{n_0}$}. 
\label{eq:defvis}
\end{equation}
The set of probability distributions of this form, for all $\bfW$ and $\bfB$, is the DBM probability model with a visible layer of $n_0$ units and $L$ hidden layers of $n_1,\ldots, n_L$ units. 
Note that every distribution of the form~\eqref{eq:defvis} is strictly positive. 

In the case that the network has only one hidden layer, $L=1$, as illustrated in the lower right panel of Figure~\ref{figure:DBM}, 
the model reduces to a restricted Boltzmann machine (with $n_0$ visible and $n_1$ hidden units). 
The corresponding set of probability distributions is denoted $\RBM_{n_0,n_1}\equiv\DBM_{n_0,n_1}$. 
If we replace all interactions, except those between the top to layers, by interactions directed towards the bottom layer, we obtain a DBN, illustrated in the upper right panel of Figure~\ref{figure:DBM}. 
We provide more details on RBMs and DBNs in the Supplementary Material. 

\section{Universal Approximation}
\label{sec:universalapproximation}

A set $\Mcal$ of probability distributions on $\{0,1\}^n$ is called {\em universal approximator} when, 
for any distribution $q$ on $\{0,1\}^n$ and any $\epsilon >0$, 
there is a distribution $p$ in $\Mcal$ with $D(q\| p)\leq \epsilon$. 
Here the Kullback-Leibler divergence between $q$ and $p$ is defined as $D(q\| p) : = \sum_{\bfx} q(\bfx) \log \frac{q(\bfx)}{p(\bfx)}$. 
This is never negative and is only zero if $q=p$. 

The main result of this paper is the following: 
\begin{theorem}
	\label{theorem}
	A DBM with a visible layer of $n$ units and $L$ hidden layers of $n$ units each is a universal approximator of probability distributions on the states of the visible layer, 
	provided $L$ is large enough. 
	More precisely, for any $n\leq n':= 2^{k} +k+1$, for some $k\in\mathbb{N}$, 
	a sufficient condition is $L\geq \frac{2^{n'}}{2(n'-\log_2(n')-1)}$. 
	For any $n$ a necessary condition is 
	$L\geq \frac{2^n -(n+1)}{n(n+1)}$.  
\end{theorem}

A direct implication of this result is the universal approximation property for conditional probability distributions of a subset of visible units, given the states of the remaining visible units. 

\begin{corollary}
	\label{corollary:stochastic}
	A DBM with a visible layer of $n$ units and $L$ hidden layers of $n$ units each is a universal approximator of stochastic input-output maps with inputs $(x_{0,1},\ldots, x_{0,k})\in\{0,1\}^k$ and outputs $(x_{0,k+1},\ldots, x_{0,n})\in\{0,1\}^{n-k}$, for any $1\leq k\leq n$, provided $L$ is as in Theorem~\ref{theorem}. 
\end{corollary}

We note that the number of visible units (minus one when this is even) is the smallest possible number of units in the first hidden layer of a DBM universal approximator. 

\begin{proposition}
	\label{proposition:narrow}
	A DBM with $n_0$ visible units can be a universal approximator only if 
	the first hidden layer has at least 
	$n_1\geq n_0-1$ units, when $n_0$ is even, and at least $n_1\geq n_0$ units, when $n_0$ is odd. 
\end{proposition}

Furthermore, Theorem~\ref{theorem} can be extended to softmax units with any finite number of states. 

\begin{theorem}
	\label{theoremdiscrete}
	A DBM with a visible layer of $n$ softmax $q$-valued units and $L$ hidden layers of $n$ softmax $q$-valued units each is a universal approximator of probability distributions on the states of the visible layer, provided $L$ is large enough. 
	More precisely, for any $n\leq n':= q^{k} +k+1$, for some $k\in\mathbb{N}$, a sufficient condition is 
	$L\geq 1+ \frac{q^{n'} -1}{q(q-1) (n' - \log_q(n')-1)}$. 
	For any $n$ a necessary condition is 
	$L\geq \frac{q^n -1}{n(q-1)(n(q-1)+2)}$.  
\end{theorem}

The proof of these statements is elaborated in the next two sections. 
First we discuss the compositional structure of DBMs and then we present a feedforward analysis.

\section{Compositional Structure}
\label{sec:compositional}

In this section we take a look at the compositional structure of DBMs. 
We will regard a DBM as a composition of two smaller DBMs. 
In order to describe these compositions, we use the renormalized entry-wise (Hadamard) product. 
The Hadamard product of two distributions $r, s \in\Delta_n$ is defined as 
\begin{equation*}
(r\ast s)(\bfz) := r(\bfz) s(\bfz)/ \sum_{\bfz'}r(\bfz') s(\bfz'),\quad \text{for all } \bfz\in\{0,1\}^n. 
\end{equation*} 
In this definition we assume that $r$ and $s$ have at least one non-zero entry in common, 
such that $\sum_{\bfz'}r(\bfz') s(\bfz') \neq 0$. 
We write $r\ast\Mcal:=\{r\ast s \colon s\in\Mcal \}$ for the set of Hadamard products of a probability distribution $r$ and the elements of a probability model $\Mcal$. 
The Hadamard product is a very natural operation for describing compositions of energy based models. 
Note that, if $r(\bfz) = \frac{1}{Z(f)} \exp(f(\bfz))$ and $s = \frac{1}{Z(g)} \exp(g(\bfz))$, then $(r\ast s)(\bfz) = \frac{1}{Z(f+g)} \exp(f(\bfz) + g(\bfz))$. 

\begin{figure}
	\centering
	\tikzset{node distance=2cm, auto}
	\begin{tikzpicture}[scale=0.9, every node/.style={transform shape}]
	\node (H2) {\unitlayer{x}{3}{{n_3}}{3} };
	\node (H1) [below of = H2] {\unitlayer{x}{3}{{n_2}}{2} };
	\node (V) [below of = H1] {\unitlayer{x}{3}{{n_1}}{1} };
	\draw[latexnew-latexnew, arrowhead=.2cm, line width = 1pt] (H2) to node { } (H1);
	\draw[latexnew-latexnew, arrowhead=.2cm, line width = 1pt] (H1) to node { } (V);
	\node (V1) [below of = V] {\unitlayer{x}{3}{{n_0}}{0} };			
	\draw[latexnew-latexnew, arrowhead=.2cm, line width = 1pt] (V1) to node { } (V);
	\node (ll1) [below of = V, node distance=.6cm] {};
	\node (l1) [left of = ll1, node distance=2.3cm] {};
	\node (l2) [right of = ll1, node distance=3.6cm] {};
	\node (ll2) [above of = H2, node distance=.6cm] {};
	\node (l3) [left of = ll2, node distance=2.3cm] {};
	\node (l4) [right of = ll2, node distance=3.6cm] {};
	
	\node (DBM2) [left of = l3, above of =l3, node distance =.5cm] {$\DBM^{(1)}$};
	
	\coordinate (L1) at (l1);
	\coordinate (L2) at (l2);
	\coordinate (L3) at (l3);
	\coordinate (L4) at (l4);
	
	\draw[rounded corners, thick] (L1)--(L2)--(L4)--(L3)--(L1)--cycle;			
	\node (ll3) [below of = V1, node distance=.6cm] {};
	\node (l5) [left of = ll3, node distance=3.6cm] {};
	\node (l6) [right of = ll3, node distance=2.3cm] {};
	\node (ll4) [above of = V, node distance=.6cm] {};
	\node (l7) [left of = ll4, node distance=3.6cm] {};
	\node (l8) [right of = ll4, node distance=2.3cm] {};
	
	\coordinate (L5) at (l5);
	\coordinate (L6) at (l6);
	\coordinate (L7) at (l7);
	\coordinate (L8) at (l8);
	
	\node (DBM2) [left of = l7, above of =l7, node distance =.5cm] {$\DBM^{(2)}$};
	
	\draw[fill=black!10,thick,rounded corners] (L5)--(L6)--(L8)--(L7)--(L5)--cycle;			
	
	\node (l22) [right of = ll1, node distance=2.2cm] {};
	\node (l77) [left of = ll4, node distance=2.2cm] {};
	\coordinate (L22) at (l22);
	\coordinate (L77) at (l77);
	
	\draw[thick,rounded corners] (L5)--(L6)--(L8)--(L7)--(L5)--cycle;			
	
	\node (V) [below of = H1] {\unitlayer{x}{3}{{n_1}}{1}  };
	\node (V1) [below of = V] {\unitlayer{x}{3}{{n_0}}{0} };			
	\draw[latexnew-latexnew, arrowhead=.2cm, line width = 1pt] (V1) to node { } (V);				
	
	
	\node (oo) [node distance=2.9 cm, right of =H2]{};
	\node (s) [node distance=2.9 cm, right of = V] {$s(\bfx_1)$};								
	\draw[vecArrow] (oo) to (s);
	
	\node (o) [node distance=2.9 cm, left of = V1] {};								
	\node (r) [node distance=2.9 cm, left of = V] {$r(\bfx_1)$};								
	\draw[vecArrowGray] (o) to (r);
	
	\node (ooo)[node distance = 5cm, right of =H2]{};
	\node (p)[node distance = 5cm, right of =V1]{$p(\bfx_0)$};
	\draw[vecArrow] (ooo) to (p);
	\node[fill=white] (rs)[node distance = 5cm, right of =V]{$(r\ast s)(\bfx_1)$};
	
	\draw[rounded corners, thick] (L1)--(L2)--(L4)--(L3)--(L1)--cycle;			
	\end{tikzpicture}
	\caption{Composition of an upper and a lower DBM to form a larger DBM. }\label{figure:decomposition}
\end{figure}
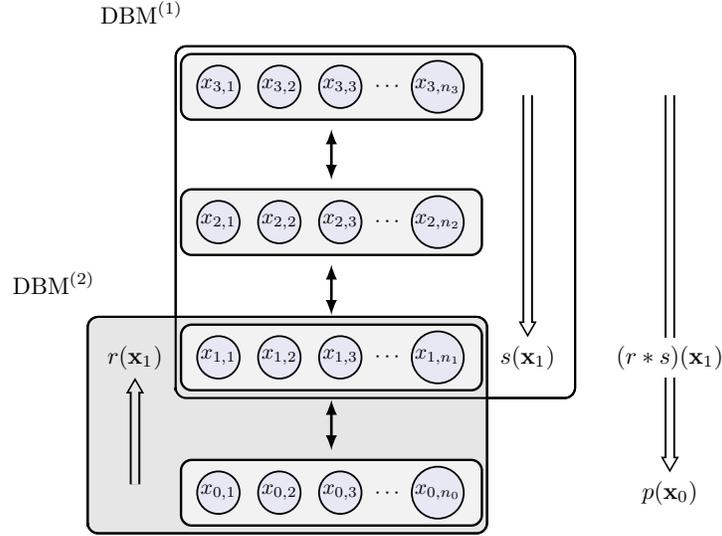

We can write the probability distributions represented by a DBM in terms of the probability distributions represented by two smaller DBMs. 
Such decompositions have been discussed previously by~\citet{Salakhutdinov:2012:ELP:2330716.2330717}. 
We identify the bottom layer of $\DBM^{(1)}$ with the top layer of $\DBM^{(2)}$, 
as illustrated in Figure~\ref{figure:decomposition}. 
By this composition, the distribution $s$ that was originally represented on the bottom layer of $\DBM^{(1)}$ becomes $r\ast s$, 
where $r$ is the distribution that was originally represented on the top layer of $\DBM^{(2)}$: 
\begin{proposition}
	\label{proposition:product}
	Consider the model $\DBM = \DBM_{n_0,\ldots, n_L}$, for some $n_0,\ldots, n_L\in\mathbb{N}$. 
	For any $0<k<L$ the marginal distributions of the $k$-th layer's units are the distributions of the form 
	\begin{equation*}
	p(\bfx_k) = (p^{(2)} \ast p^{(1)})(\bfx_k) ,\quad\text{for all $\bfx_k\in\{0,1\}^{n_k}$}, 
	\end{equation*}	
	where $p^{(1)}(\bfx_k)$ is a bottom layer marginal of $\DBM^{(1)}=\DBM_{n_k,\ldots,n_L}$ and 
	$p^{(2)}(\bfx_k)$ is a top layer marginal of $\DBM^{(2)}=\DBM_{n_0,n_1,\ldots, n_k}$. 
\end{proposition}
For completeness we provide a proof of this statement in the Supplementary Material. 

Next we define the model of conditional probability distributions represented by a feedforward layer. 
For  $n_1$ input units and $n_0$ output units, the feedforward model $\FF_{n_0,n_1}$ consists of all conditional probability distributions of the form 
\begin{equation*}
q_{\bfW_0,\bfB_0} (\bfx_0|\bfx_1) = \frac{1}{Z(\bfW_0\bfx_1 + \bfB_0)}\exp(\bfx_0^\top \bfW_0 \bfx_1 + \bfx_0^\top \bfB_0) , \quad\text{for all $\bfx_0\in\{0,1\}^{n_0}$,  $\bfx_1\in\{0,1\}^{n_1}$}.  
\label{eq:feedforwarddef}
\end{equation*}
Here $\bfW_0\in\mathbb{R}^{n_0\times n_1}$ is a matrix of input weights and $\bfB_0\in\mathbb{R}^{n_0}$ is a vector of biases. 
Clearly, these conditionals correspond exactly to the conditionals represented between first hidden layer and the visible layer of a DBM, for the same choices of parameters. 

The next Proposition~\ref{factorization} gives an expression for the visible distributions represented by a DBM in terms of the distributions represented by two smaller DBMs and the conditionals represented by a feedforward layer with shared parameters. 

\begin{proposition}\label{factorization}
	The bottom layer marginal distributions representable by $\DBM_{n_0,\ldots, n_L}$ are those of the from 
	\begin{equation*}
	p(\bfx_0) = \sum_{\bfx_1} q(\bfx_0|\bfx_1) (r \ast s)(\bfx_1), \quad\text{for all $\bfx_0\in\{0,1\}^{n_0}$},
	\end{equation*}
	where $q(\bfx_0|\bfx_1)r(\bfx_1)$ is a joint probability distribution of the fully observable $\RBM_{n_0,n_1}$
	and $s$ is a bottom layer marginal of $\DBM_{n_1,\ldots, n_L}$. 
\end{proposition}
\begin{proof}[Proof of Proposition~\ref{factorization}]
	We have 
	\begin{equation*}
	p(\bfx_0) = \sum_{\bfx_1} p(\bfx_0|\bfx_1) p(\bfx_1), \quad\text{for all $\bfx_0\in\{0,1\}^{n_0}$}. 
	\end{equation*}
	By Proposition~\ref{proposition:product}, $p(\bfx_1) = (r\ast s)(\bfx_1)$ for all $\bfx_1\in\{0,1\}^{n_1}$. 
\end{proof}

The proposition is illustrated in Figure~\ref{figure:decomposition}. 
Note that $r(\bfx_1)$ is a top layer marginal of $\RBM_{n_0,n_1}$ and $q(\bfx_0|\bfx_1)$ is the top-to-bottom conditional of $\RBM_{n_0,n_1}$ corresponding to the feedforward layer $\FF_{n_0,n_1}$. 
Proposition~\ref{factorization} suggests that it is possible to study the representational power of DBMs in terms of the representational power of smaller DBMs composed with simple feedforward networks. 
The problem is that the distribution $r\ast s$, intended as the input of the feedforward layer, 
depends on the same parameters $\bfW_0, \bfB_0$ as the feedforward layer. 
Hence the input of the feedforward layer cannot be chosen independently from the transformation that the feedforward layer applies on it. 
Nonetheless, as we will show in the next section, it is possible to resolve this difficulty and analyze the representational power of the DBM in a sequential way.

\section{Feedforward Analysis}
\label{sec:analysis}

Consider a DBM composed of an upper and a lower part, as shown in Figure~\ref{figure:decomposition}. 
If the upper $\DBM^{(1)}$ is able to ``disable'' or neutralize the top layer marginal $r$ of $\DBM^{(2)}$, 
then the distribution represented at the bottom layer of the compound DBM can be regarded as the feedforward pass of the distribution $s$ represented at the bottom layer of $\DBM^{(1)}$. 
Namely, by Proposition~\ref{factorization} the visible distribution of the combined network is the result of passing the marginal distribution $(r\ast s)(\bfx_1)$ feedforward through the conditional distribution $q(\bfx_0 | \bfx_1)$.

\subsection{Disabling the backward signal}

In order to make the feedforward approach work, we need to resolve the problem that the marginal $r$ and the conditional $q$ share the same parameters. 
When we modify these parameters in order to obtain a specific conditional $q$ representing a desired feedforward transformation, 
the marginal $r$ changes as well, and with it also the input $r\ast s$. 
We resolve this dilemma in the following way. 
Instead of regarding the bottom layer marginals of $\DBM^{(1)}$ as the input model, 
we restrict our attention to a subset $\Gcal$ of the bottom layer marginals of $\DBM^{(1)}$ with the following property: 
\begin{equation}
r \ast \Gcal = \Gcal \quad\text{for all top layer marginals $r$ of $\DBM^{(2)}$}. 
\label{eq:stable}
\end{equation}
In this case, any desired input $s\in\Gcal$, together with any desired conditional $q\in\FF_{n_0,n_1}$, can be obtained by the following procedure: 
\begin{enumerate}
	\item Tune the parameters of $\DBM^{(2)}$ to represent any desired (representable) conditional distribution $q$. 
	By tuning the parameters in this way, the top layer marginal of $\DBM^{(2)}$ becomes a distribution $r$ that depends on $q$. 
	\item Tune the parameters of $\DBM^{(1)}$ to represent a bottom layer marginal $s'\in\Gcal$ with $r\ast s' =s$. 
\end{enumerate}
Now we just need to find a good choice of $\Gcal$, from which we require the following. 
\begin{itemize}
	\item 
	The set $\Gcal$ has to satisfy~\eqref{eq:stable}. 
	\item
	We have to make sure that $\Gcal$ is contained in, or can be approximated arbitrarily well, by the distributions representable at the bottom layer of $\DBM^{(1)}$. 
	\item 
	Furthermore, $\Gcal$ should be as large as possible, in order to account for the largest possible fraction of the representational power of $\DBM^{(1)}$. 
\end{itemize}

We choose $\Gcal$ as the set of probability distributions on $\{0,1\}^{n_1}$ 
that assign positive probability only to a subset of vectors $S\subseteq\{0,1\}^{n_1}$, i.e., as the set  
\begin{equation*}
\Delta_{n_1}(S):=\{ p\in\Delta_{n_1} \colon p(\bfx_1)=0\text{ for all } \bfx_1\not\in S  \}. 
\end{equation*}  
In the next Proposition~\ref{kill} we show that this set satisfies the first item of the list, regardless of $S$. 
For the second and third items, we have to choose $S$ depending on the size of $\DBM^{(1)}$. 
We will discuss the details of this further below, in Section~\ref{sec:feedforwardanalysis}. 

Given a set of probability distributions $\Mcal\subseteq\Delta_n$, let $\overline{\Mcal}\subseteq\Delta_n$ denote the set of probability distributions that can be approximated arbitrarily well by elements from $\Mcal$. 
\begin{proposition}\label{kill}
	Let $r\in\Delta_{n}$ be a strictly positive probability distribution and let $\Mcal\subseteq\Delta_n$ be a set of probability distributions with $\overline{\Mcal}\supseteq\Delta_n(S)$. 
	Then $\overline{r\ast\Mcal}\supseteq\Delta_n(S)$. 
\end{proposition}

\begin{proof}[Proof of Proposition~\ref{kill}]
	Since $\Mcal$ can approximate any distribution from $\Delta_n(S)$ arbitrarily well, it can approximate any distribution of the form $s'(\bfz) = (s / r)(\bfz):= (s(\bfz)/r(\bfz) ) \frac{1}{\sum_{\bfz'}  s(\bfz')/r(\bfz') }$, $\bfz\in\{0,1\}^n$, arbitrarily well, where $s$ is any distribution from $\Delta_n(S)$. 
	Any such $s'$ is contained in $\Delta_n(S)$, as it is strictly supported on $S$. 
	Now, the Hadamard product of $r$ and $s'$ is given by 
	\begin{align*}
	(r\ast s')(\bfz) 
	&= r(\bfz) s'(\bfz) \frac{1}{\sum_{\bfz'}r(\bfz')s'(\bfz')} \\
	&= r(\bfz) (s(\bfz) / r(\bfz)) \frac{1}{\sum_{\bfz''}s(\bfz'')/r(\bfz'')} \frac{1}{\sum_{\bfz'} r(\bfz')s'(\bfz')}\\
	&= s(\bfz) \frac{1}{\sum_{\bfz''}s(\bfz'')/r(\bfz'')} \frac{1}{\sum_{\bfz'} r(\bfz')  (s(\bfz') / r(\bfz')) \frac{1}{\sum_{\bfz'''}s(\bfz''')/r(\bfz''')}}\\
	&= s(\bfz) \frac{1}{\sum_{\bfz''}s(\bfz'')/r(\bfz'')} \frac{\sum_{\bfz'''}s(\bfz''')/r(\bfz''')}{\sum_{\bfz'}  s(\bfz')}\\
	&= s(\bfz), \quad\text{for all $\bfz\in\{0,1\}^n$}.  
	\end{align*}
	Since $s$ was an arbitrary distribution from the set $\Delta_n(S)$, this proves the claim. 
\end{proof}

\subsection{Proof of Theorem~\ref{theorem}}
\label{sec:feedforwardanalysis}

In the previous subsection we have shown that DBMs can be studied from a feedforward perspective. 
Let us make this more explicit. 
Putting Propositions~\ref{factorization} and~\ref{kill} together, we arrive at: 
\begin{proposition}
	\label{proposition:directed}
	If $\DBM_{n_1,\ldots, n_L}$ can approximate every distribution from the set $\Delta_{n_1}(S)$ arbitrarily well as its bottom layer marginal, 
	then $\DBM_{n_0,n_1,\ldots, n_L}$ can approximate every distribution from the set $\FF_{n_0,n_1} (\Delta_{n_1}(S))$ arbitrarily well as its bottom layer marginal. 
\end{proposition} 

With this proposition, we can study the representational power of DBMs sequentially, increasing from layer to layer. 
A feedforward layer is able to compute many interesting transformations of its input. 
For any choice of parameters, the conditional distribution $q_{\bfW_0,\bfB_0}$ represented by the feedforward layer $\FF_{n_0,n_1}$ defines a map 
$\Delta_{n_1}\to\Delta_{n_0}$ taking a probability distribution $p$ to a probability distribution $\sum_{\bfx_1} p(\bfx_1) q_{\bfW_0,\bfB_0}(\bfx_0 | \bfx_1)$.  
As we vary the parameters $\bfW_0, \bfB_0$, every input distribution $p$ is mapped to a collection of output distributions. 
Hence the feedforward layer can augment the representational power of the input model. 
After a sufficient number of feedforward layers, the output distribution can be made to approximate any desired probability distribution arbitrarily well. 

We focus on the DBM with layers of constant size $n$. 
First, we need to show that a DBM with $n$ visible units and $l$ hidden layers of $n$ units each can approximate any distribution from $\Delta_n(S^l)$ arbitrarily well, for some $S^l\subseteq\{0,1\}^n$. 
Then, we need to show that by transformations with a feedforward layer, we obtain a larger set $\Delta(S^{l+1}) \subseteq\FF_{n,n}(\Delta(S^l))$, 
which in turn can be approximated arbitrarily well by the DBM with $l+1$ hidden layers. 
The idea is to obtain an increasing sequence 
\begin{equation*}
S^1\subset S^2\subset S^3\subset \cdots\subset S^L=\{0,1\}^n,
\end{equation*} 
meaning that the DBM with $L$ hidden layers can approximate any distribution on $S^L = \{0,1\}^n$ arbitrarily well. 

We start with $l=1$. The representational power of RBMs has been studied in previous papers. 
We use the following Proposition~\ref{proposition:RBM}~\citep[taken from][]{Montufar2011}. 
We call a pair of states $\bfx,\bfx' \in \{0,1\}^n$ {\em adjacent} if their Hamming distance is one, i.e., $d_H(\bfx,\bfx'):=|\{i\in[n]\colon x_i\neq x'_i \}|=1$. 

\begin{proposition}
	\label{proposition:RBM}
	The model $\RBM_{n_0,n_1}$ can approximate every distribution from $\Delta_{n_0}(S)$ arbitrarily well as its bottom marginal, 
	where $S\subseteq\{0,1\}^{n_0}$ is any union of $n_1 + 1$ pairs of adjacent states. 
\end{proposition}

Feedforward layers have been studied in previous papers as well and we can take advantage of the tools that have been developed there. 
The following Proposition~\ref{proposition:sharing}~\citep[taken from][]{Montufar2014dbn} describes the augmentation of an input set $\Delta_n(S)$ to an output set $\Delta_n(S\cup P)$. 
The {\em flip} of a state vector $\bfx$ along $j$ is the vector $\bfx_{\bar j}$ that results from inverting the $j$-th entry of $\bfx$. 
\begin{proposition} 
	\label{proposition:sharing}
	The image of $\Delta_{n}(S)$ by $\FF_{n,n}$ can approximate every distribution from $\Delta_{n}(S\cup P)$ arbitrarily well, 
	where $P\subseteq\{0,1\}^n$ is any set of the following form. 
	Take $n$ disjoint pairs of adjacent states $p^1,\ldots, p^{n}$ and $n$ distinct directions $i_1,\ldots,i_{n}$. 
	Intersect each pair $p^j$ with $S$ and flip the result along the direction $i_j$, 
	to obtain $\bar p^1 =(S\cap p^1)_{\bar i_1},\ldots, \bar p^n =(S\cap p^n)_{\bar i_n}$. 
	Set $P=\{\bar p^1,\ldots, \bar p^n\}$. 
\end{proposition}

\citet{Montufar2011} show that, for any $k\in\mathbb{N}$ and $n=2^k + k + 1$, 
there is a choice of  $S^1$ of the form described in Proposition~\ref{proposition:RBM} (e.g., the set of all length-$n$ strings whose last $2^k$ entries are zero),  
and a sequence 
$S^2 = S^1\cup P^1,\ldots, S^L = S ^{L-1}\cup P^{L-1}$ of the form described in Proposition~\ref{proposition:sharing},  
such that $S^L = \{0,1\}^n$ for $L = \frac{2^{n-1}}{2^k}$. 
This implies the existence and sufficiency statements from Theorem~\ref{theorem}. 
The necessary condition results from straightforward parameter counting arguments; comparing the dimension $\dim(\Delta_n) = 2^n-1$ of the set being approximated and the number of parameters $L n^2 + (L+1)n$ of the DBM. 
This concludes the proof of Theorem~\ref{theorem}. 
Details on the other statements are given in the Supplementary Material. 

\section{Conclusion}
\label{sec:discussion}

This paper proves that undirected layered deep networks are, in a well defined sense, as powerful as their feedforward counterparts. 
We see this as an important contribution to developing better intuitions about the advantages and disadvantages of using feedforward vs. undirected architectures. The methods developed in this paper seem valuable for studying the effects of training undirected networks sequentially, or using the trained weights of a DBN to initialize a DBM. 

This paper proves the universal approximation property for narrow DBMs thereby settling an intuition that had been missing formal verification for a surprisingly long time. 
This complements previously known results addressing RBMs and narrow DBNs, which can be regarded the shallow and feedforward counterparts of narrow DBMs. 
%
%
We investigated the compositional structure of DBMs and presented a trick to separate the activities on the upper part of the network from those on the lower part of the network. 
This allowed us to trace parameter regions where DBMs can be regarded as feedforward networks, 
passing the probability distributions represented at the higher layers downwards from layer to layer by multiplication with independent conditional probability distributions which have the same form as those represented by feedforward layers. 
%


\subsubsection*{Acknowledgments}
I am grateful to the Santa Fe Institute, where I was hosted while working on this article.

\bibliography{referenzen}
\bibliographystyle{iclr2015}

\appendix

\section*{Supplementary Material}
\label{sec:appendix}

\subsection*{Proof of Proposition~\ref{proposition:product}}
The $k$-th layer marginal of the DBM satisfies
	\begin{eqnarray*}
		p(\bfx_k) 
		&=& \sum_{\bfx_0,\ldots, \bfx_{k-1},\bfx_{k+1}, \ldots, \bfx_L} p(\bfx_0,\bfx_1,\ldots, \bfx_L)\\
		&=& \sum_{\bfx_0,\ldots, \bfx_{k-1},\bfx_{k+1}, \ldots, \bfx_L}  \frac{1}{Z(\bfW,\bfB)}
		\exp\Big( \sum_{l=0}^{L-1}  \bfx_{l}^\top \bfW_{l} \bfx_{l+1}  + \sum_{l=0}^L \bfx_l^\top {\bfB_l}\Big)\\
		&=& \sum_{\bfx_0,\ldots, \bfx_{k-1},\bfx_{k+1}, \ldots, \bfx_L}  \frac{1}{Z(\bfW,\bfB)}
		\exp\Big( \sum_{l=0}^{k-1}  \bfx_{l}^\top \bfW_{l} \bfx_{l+1}  + \sum_{l=0}^{k-1} \bfx_l^\top {\bfB_l} + \bfx_k^\top {\bfB'_k}\Big)\\
		&&\phantom{\sum_{\bfx_0,\ldots, \bfx_{k-1},\bfx_{k+1}, \ldots, \bfx_L}  \frac{1}{Z(\bfW,\bfB)}} 
		\times\exp\Big( \sum_{l=k}^{L-1}  \bfx_{l}^\top \bfW_{l} \bfx_{l+1}  + \sum_{l=k}^L \bfx_l^\top {\bfB_l} -\bfx_k^\top \bfB'_k  \Big)\\
		&=& \frac{1}{Z(\bfW,\bfB)}
		\sum_{\bfx_0,\ldots, \bfx_{k-1}} 
		\exp\Big( \sum_{l=0}^{k-1}  \bfx_{l}^\top \bfW_{l} \bfx_{l+1}  + \sum_{l=0}^{k-1} \bfx_l^\top {\bfB_l} + \bfx_k^\top {\bfB'_k}\Big)\\
		&&\phantom{\frac{1}{Z(\bfW,\bfB)}}
		\times\sum_{\bfx_{k+1}, \ldots, \bfx_L}
		\exp\Big( \sum_{l=k}^{L-1}  \bfx_{l}^\top \bfW_{l} \bfx_{l+1}  + \sum_{l=k}^L \bfx_l^\top {\bfB_l} -\bfx_k^\top \bfB'_k  \Big)\\
		&=& \frac{1}{Z(\bfW,\bfB)} \; Z(\bfW^{(2)}, \bfB^{(2)} ) p^{(2)}(\bfx_k) \; Z(\bfW^{(1)} ,\bfB^{(1)})  p^{(1)}(\bfx_k), 
		\quad \text{for all $\bfx_k\in\{0,1\}^{n_k}$}. 
	\end{eqnarray*}
	This shows that for any $k$-th layer marginal $p(\bfx_k)$ representable by $\DBM$, 
	there is a distribution $p^{(2)}(\bfx_k)$ representable as the top layer marginal of $\DBM^{(2)}$ with parameters $\bfW^{(2)}= \{ \bfW_0,\ldots, \bfW_{k-1} \}$, $\bfB^{(2)} = \{ \bfB_0,\ldots, \bfB_{k-1}, \bfB'_k\}$, 
	and a distribution $p^{(1)}(\bfx_k)$ representable as the bottom layer marginal of $\DBM^{(1)}$ with parameters $\bfW^{(1)}=\{ \bfW_{k},\ldots, \bfW_{L-1}\}$,  $\bfB^{(1)}=\{ \bfB_k-\bfB'_k,\bfB_{k+1},\ldots,\bfB_L\}$, 
	such that the equation $p(\bfx_k) = (p^{(2)}\ast p^{(1)})(\bfx_k)$ holds, and vice versa. \qed

\subsection*{Approximation of stochastic maps}

A stochastic map with inputs $\{0,1\}^k$ and outputs $\{0,1\}^m$ assigns a probability distribution $p(\cdot| \bfi )\in\Delta_{m}$ to each input vector $\bfi\in\{0,1\}^{k}$. 
DBMs can be used to define such maps by clamping the states of some of their units to the input values $\bfi$, 
and taking the resulting conditional probability distribution over the states of some other units as the output distributions. 
One way of doing this is by dividing the visible units in two groups, corresponding to inputs and outputs, as $\bfx_0 = (\bfi,\bfo)$. 
Given that $p(\bfx_0) = p(\bfi,\bfo)$ stands in one to one relation to the pair $(p(\bfi), p(\bfo|\bfi))$, 
Corollary~\ref{corollary:stochastic} is a direct implication of Theorem~\ref{theorem}.  
%

Note that a universal approximator of stochastic maps is also a universal approximator of deterministic maps. 
Every deterministic map $\bfi \mapsto \bfo=f(\bfi)$ can be regarded as the special type of stochastic map $\bfi \mapsto \delta_{f(\bfi)}(\bfo)$, where $\delta_{f(\bfi)}$ is the Dirac delta assigning probability one to $\bfo = f(\bfi)$. 

Corollary~\ref{corollary:stochastic} complements previous results addressing universal approximation of stochastic maps by conditional RBMs~\citep{Maaten2011, montufar2014expressive}. 
As discussed in~\citep{montufar2014expressive}, in contrast to joint probability distributions, stochastic maps do not need to model the input distributions, 
and hence universal approximators of stochastic maps need not be universal approximators of joint probability distributions. 
It would be interesting to investigate corresponding refinements of Corollary~\ref{corollary:stochastic} in future work.  

\subsection*{Softmax units}

All arguments presented in the main part of this article hold for arbitrary finite valued units (not only binary units). 
An analysis of sequences of feedforward layers of finite valued units is available from~\citep{Montufar2014dbn}. 
This allows us to formulate Theorem~\ref{theoremdiscrete} as a direct generalization of Theorem~\ref{theorem}. 
The result can be further refined to cases where each layer has units with different numbers of possible states. We omit further details at this point. 

\subsection*{Minimal width of universal approximators}

In a layered network, a too narrow layer represents a bottleneck. 
It is an interesting question how narrow a universal approximator can be. 
Proposition~\ref{proposition:narrow} shows that if the visible layer has $n_0$ units, then the first hidden layer of a universal approximator must have at least $n_1\geq n_0-1$ units.  
In fact, when $n_0$ is odd, this has to be at least $n_1\geq n_0$. 

\begin{proof}[Proof of Proposition~\ref{proposition:narrow}]
	This follows from the fact that the visible distributions of the DBM are mixtures of the conditionals $p(\bfx_0|\bfx_1)$, for all $\bfx_1\in\{0,1\}^{n_1}$. 
	Each of these conditional distributions is a product distribution. 
	There are distributions on $\{0,1\}^{n_0}$ that can only be approximated by mixtures of product distributions, 
	if these mixtures involve mixture components that approximate all point measures assigning probability one to the binary strings with an odd number of ones~\citep[see][]{Montufar2010a}. 
	
	Now,~\citet[][Proposition~3.19]{MontufarMorton2012} show that when $n_0$ is odd, there is no $(n_0-1)$-generated zonoset with a point in each odd (or each even) orthant of $\mathbb{R}^{n_0}$. 
	Without going into more details, this implies that, when $n_1=n_0-1$, with odd $n_0$, 
	the set of conditionals $\{ p(\bfx_0|\bfx_1) \colon \bfx_1\in\{0,1\}^{n_1}\}$ cannot approximate the set of point measures that assign probability one to the binary strings with an odd (or even) number of ones. 
\end{proof}
We note that the same width bound holds for DBNs, since the visible distributions represented by DBNs are mixtures of the same product distributions as the visible distributions of DBMs. 

\subsection*{Comparison with narrow DBNs}

DBNs have the same network topology as DBMs, but with interactions directed towards the bottom layer, 
except for the interactions between the deepest two layers, which are undirected. 
The corresponding joint probability distributions have the form 
\begin{multline}
p_{\bfW,\bfB}(\bfx_0,\bfx_1\ldots,\bfx_L) = 
p_{\bfW_{L-1},\bfB_{L-1},\bfB_L} (\bfx_{L-1},\bfx_L)  
\prod_{l=0}^{L-2} p_{\bfW_{l},\bfB_l}( \bfx_l | \bfx_{l+1} ), \\
\quad\text{for all $(\bfx_0,\ldots, \bfx_L)\in\{0,1\}^{n_0+\cdots+n_L}$}. 
\label{eq:defjointdbn}
\end{multline}
Here the distributions of the states in the deepest two layers are given by  
\begin{multline}
p_{\bfW_{L-1},\bfB_{L-1},\bfB_L} (\bfx_{L-1},\bfx_L)  
=  \frac{1}{Z(\bfW_{L-1},\bfB_{L-1}, \bfB_L)} \exp( \bfx_{L-1}^\top \bfW_{L-1} \bfx_{L}  + \bfx_{L-1}^\top {\bfB_{L-1}} + \bfx_L^\top \bfB_L ), \\ 
\quad\text{for all $(\bfx_{L-1},\bfx_L)\in\{0,1\}^{n_{L-1}+n_L}$}.  
\end{multline}
The conditional distributions (feedforward layers), are given by 
\begin{multline}
p_{\bfW_{l},\bfB_l}( \bfx_l | \bfx_{l+1} ) 
= \frac{1}{Z(\bfW_{l}\bfx_{l+1},\bfB_{l})} \exp( \bfx_{l}^\top \bfW_{l} \bfx_{L}  + \bfx_{l}^\top {\bfB_{l}}), \\ 
\quad\text{for all $\bfx_{l} \in\{0,1\}^{n_l}$, for all $\bfx_{l+1}\in\{0,1\}^{n_{l+1}}$}. 
\end{multline}

Although DBNs have undirected interactions between the top two layers, in the narrow case the universal approximation capability stems essentially from the feedforward part. 
A DBN with layers of width $n$ is a universal approximator if the number of hidden layers satisfies $L\geq \frac{2^n}{2(n-\log_2(n)-1)}$ and only if $L\geq \frac{2^n -(n+1)}{n(n+1)}$~\citep{Montufar2011}. 
These bounds correspond exactly to the bounds we obtained in Theorem~\ref{theorem} for DBMs. 
In our proof we showed that the kinds of transformations of probability distributions exploited in~\citep{Montufar2011} in the context of DBNs can also be represented by DBMs. 
In particular, our analysis shows that many distributions that are representable by DBNs are also representable by DBMs of the same size. 

\subsection*{Comparison with RBMs}

In the case of one single hidden layer, the DBM reduces to an RBM. 
RBMs are universal approximators, provided the hidden layer contains sufficiently many units. 
The minimal number of hidden units $m$ for which an RBM with $n$ visible units is a universal approximator is at least $\frac{2^n-n}{n+1}$ and at most $2^{n-1}-1$~\citep{Montufar2011}. 
The exact value is not known, but there are examples where the lower bound is not attained. 
For narrow DBMs we obtained an upper bound on the minimal number of layers sufficient for universal approximation of the form $L\geq 2^n/ 2(n - \log_2(n) -1)$. 
Hence both RBMs and narrow DBMs require at most a number of interaction weights and biases of order $O (n 2^{n-1}) $. 
We should note that in both cases, it is possible to formulate restrictions on the interaction weights and biases in such a way that the total number of free parameters needed for universal approximation is $2^n-1$, i.e.,  just as large as the dimension of the set~$\Delta_n$. 

\subsection*{Exploiting the backward activity}

The product $r\ast s$ arising in Proposition~\ref{factorization} can be used to augment the input model that is passed to the feedforward layer. 
As long as this does not interfere with the choice of a desirable conditional $q$, this could be exploited to obtain a more compact construction of a universal approximator. 
Investigating this in detail could help us better understand the differences of DBNs and DBMs. It would be interesting to take a closer look at this in future work.

\end{document}